\documentclass[letterpaper, 10 pt, conference]{ieeeconf}  

\IEEEoverridecommandlockouts                              

\usepackage{graphics} 
\usepackage{epsfig} 
\usepackage{mathptmx} 
\usepackage{times} 

\usepackage{amsmath} 
\usepackage{amssymb}  
\usepackage{algorithm}
\usepackage{ulem}
\usepackage{amsmath}
\usepackage[noend]{algpseudocode}
\usepackage{xcolor}
\usepackage{subfigure} 
\newtheorem{theorem}{Theorem}[section]

\newtheorem{lemma}[theorem]{Lemma}
\newtheorem{assumption}{Assumption}

\ifodd 1
\else

\fi

\title{\LARGE \bf
DeepSafeMPC: Deep Learning-Based Model Predictive Control for Safe Multi-Agent Reinforcement Learning
}

\author{Xuefeng Wang$^{1}$, Henglin Pu$^{2}$, Hyung Jun Kim$^{1}$  and Husheng Li$^{1, 2  \dag} $
\thanks{*This work was supported by the National Science Foundation under grants 2135286, 2109295 and 2128455.}
\thanks{$^{1}$Xuefeng Wang, Hyung Jun Kim and Husheng Li are with School of Aeronautics and Astronautics, Purdue University, West Lafayette, IN 47907, USA
        {\tt\small wang6067@purdue.edu}}%
\thanks{$^{2}$Henglin Pu and Husheng Li are with the Elmore Family School of Electrical and Computer Engineering, Purdue University, West Lafayette, IN 47907, USA
        {\tt\small pu36@purdue.edu}}%
\thanks{\dag: Corresponding Author}
}

\begin{document}

\maketitle
\thispagestyle{empty}
\pagestyle{empty}

\begin{abstract}

Safe Multi-agent reinforcement learning (safe MARL) has increasingly gained attention in recent years, emphasizing the need for agents to not only optimize the global return but also adhere to safety requirements through behavioral constraints.
Some recent work has integrated control theory with multi-agent reinforcement learning to address the challenge of ensuring safety. However, there have been only very limited applications of Model Predictive Control (MPC) methods in this domain, primarily due to the complex and implicit dynamics characteristic of multi-agent environments. To bridge this gap, we propose a novel method called  Deep Learning-Based Model Predictive Control for Safe Multi-Agent Reinforcement Learning (DeepSafeMPC). The key insight of DeepSafeMPC is leveraging a centralized deep learning model to well predict environmental dynamics. 
Our method applies MARL principles to search for optimal solutions. 
Through the employment of MPC, the actions of agents can be restricted within safe states concurrently.
We demonstrate the effectiveness of our approach using the Safe Multi-agent MuJoCo environment, showcasing significant advancements in addressing safety concerns in MARL.
\end{abstract}

\section{INTRODUCTION}

The burgeoning field of multi-agent reinforcement learning has garnered considerable attention  \cite{marl_survey} due to its potential to solve a wide range of complex decision-making problems such as robotics \cite{ma_robotics}, path planning \cite{primal}, and games \cite{alpha_Star}. By leveraging the capabilities of multiple learning agents, MARL systems are poised to tackle tasks that are too intricate for solitary agents to handle effectively. However, most of existing works predominantly focus on optimizing returns without adequately addressing the safety concerns such like colliding with other robotics and autonomous cars  \cite{safety}. This oversight raises the risk of agents engaging in behaviors that could lead to catastrophic outcomes.

To address the safety issue \cite{safe_review}, safe reinforcement learning approaches have emerged, enabling the mitigation of risks associated with autonomous agents in real-time and in the real-world scenarios \cite{nips_1, nips_2, nips_3}. 
However, it is still a daunting challenge to develop safe policies for multi-agent systems due to
the presence of multiple agents that introduces non-stationarity \cite{non_stationary} to the environment. Moreover, when agents are compelled to factor in the safety constraints of their counterparts to achieve collective convergence toward a shared secure zone, 
the optimization process for the entire system becomes markedly intricate. This intricacy renders the overall optimization process highly challenging.

To tackle the aforementioned challenges, a fusion of reinforcement learning and control methodologies has emerged. For instance, researchers have delved into the utilization of Lyapunov functions to ensure stability and safety \cite{lyapunov}. However, these Lyapunov functions are typically handcrafted, posing difficulties in construction and lacking generalization. 
Moreover, some efforts aim to employ robust control \cite{robust_control} to ensure the stability in safe RL. However, direct application to complex multi-agent dynamics is hindered because it focuses on linear models.
Additionally, adaptive control has been investigated \cite{adaptive_control}, albeit limited by its ability to address uncertainty effectively.
Among these methodologies, MPC stands out for its potential in facilitating resilient decision-making. This control strategy utilizes a model of the system to make predictions about future states, allows for the optimization of control actions over a future horizon while considering constraints which particularly appealing in the realm of safe RL due to its forward-looking nature.
Nonetheless, while certain safe RL strategies adopt MPC methods to enhance decision-making \cite{safe_mpc, safe_mpc_2}, their effectiveness is diminished in real-world scenarios due to the failure of accounting for the environment's implicit dynamics.

Deriving insights from deep learning-based MPC \cite{deepmpc}, we develop an innovative approach that integrates MPC into addressing safe MARL problems. The resulting algorithm, named DeepSafeMPC, follows a two-step process to provide a solution to this challenge. Initially, leveraging the MARL paradigm, we employ Multi-Agent Proximal Policy Optimization (MAPPO) \cite{mappo} to explore and optimize the policies of the agents. Subsequently, utilizing the policy generated by reinforcement learning as an initial guess, we integrate it into MPC to minimize the overall cost function while maintaining the effectiveness of policy.

This paper endeavors to bridge the gap between theoretical safety considerations in MARL and the practical deployment of robust systems, thereby facilitating more dependable and resilient applications of reinforcement learning in multi-agent scenarios. Our contributions are summarized as follows:

\begin{itemize}
\item We introduce the framework of DeepSafeMPC, a novel method aims to address safety concerns within intricate multi-agent environments, thereby enhancing the reliability of the multi-agent systems.
\item We integrate MPC and MARL paradigms, utilizing MAPPO algorithms to explore and optimize the global return while harnessing the capabilities of MPC to navigate stringent safety constraints effectively.
\item We utilize the deep learning model to accurately forecast implicit environmental dynamics, enhancing MPC's efficiency in minimizing the cost function shaped by safety criteria, thus ensuring robust decision-making.
\item We conduct our simulations within the Safe Multi-agent MuJoCo environment, demonstrating the effectiveness in mitigating safety concerns in the Safe MARL setups.
\end{itemize}

\section{RELATED WORK}
Numerous studies have integrated advanced control theories into Safe MARL. Adaptive control, for instance, targets systems with uncertain parameters, adjusting the controller or model in real-time to enhance performance \cite{adaptive}. Though manually designing Lyapunov functions poses a significant challenge in establishing systematic principles for ensuring agent safety and performance \cite{lynapnov_4}, Lyapunov-based approaches ensure stability and near-constraint satisfaction with each policy update \cite{lynapnov_2, lynapnov_3}. 
Additionally, robust control, a design methodology ensuring stability against predefined bounded disturbances, addresses unknown dynamics and noise effectively. Unlike adaptive control, which adjusts the current parameters, robust control identifies a suitable controller for all potential disturbances and maintains it without further adjustments.

In recent developments, Safe RL methods have been proposed from the MPC standpoint. These approaches have integrated MPC algorithms with the reinforcement learning framework, while they typically treat dynamics as linear models, which is impractical in real-world situations \cite{MPC_1}. Considering multi-agent systems, the dynamics of environments are typically non-linear and highly implicit. Thus, accurately predicting future states becomes essential. While research has been somewhat confined to single-agent scenarios, there has been an investigation into this issue, utilizing deep learning to construct predictive models to address the implicit dynamics \cite{deepmpc, lstm, mpc_learning}. To tackle with the implicit dynamics in multi-agent environments, our method aims to leverage a centralized deep learning-based predictor model to well predict the future states, making it more suitable for multi-agent reinforcement learning applications.

\section{PRELIMINARIES}
\subsection{Constrained Markov Decision Process}
We formalize the problem within the framework of a decentralized partially observable Markov decision process (Dec-POMDP) \cite{pomdp}. A Dec-POMDP is defined by a tuple \( M = (S, A, P, R, \Omega, O, n, \gamma) \), where \( n \) is the number of agents, \( \gamma \in [0,1) \) represents the discount factor, \( S \) denotes the finite state space, and \( \Omega \) is the set of observations. Each agent \( i \) receives a local observation \( o_i \) from the observation function \( O(s, a) \). Each agent selects an action \( a_i \in A \), contributing to a joint action \( \mathbf{a} \in A^n \), which leads to the next state \( s' \) as dictated by the state transition function \( P(s' | s, \mathbf{a}) \). The agents collectively aim for a global reward \( R(s, \mathbf{a}) \), which is determined by the reward function \( r = R(s, a) \). The action-observation history for each agent at time step \( t \) is represented by \( \tau^t \in (\Omega \times A)^t \).

A Constrained Markov Decision Process (CMDP) is a MDP with an additional set of constraints \( C \) that limit the set of allowable policies. The set \( C \) comprises cost functions \( C_k: S \times A \rightarrow \mathbb{R} \), for \( k=1, \ldots, n \), and the \( C \)-return as \( J_C(\tau) = \mathbb{E}_{\tau}\left[\sum_{t=0}^{\infty} \gamma^t C(s^t, a^t)\right] \). The set of admissible policies is then \( \Pi_C = \{ \pi \in \Pi: J_C(\tau) \leq b_i, \forall i \} \), where \( b_i \) denotes the cost's upper bound. The reinforcement learning objective with respect to a CMDP is to discover a policy \( \pi^* \) that maximizes the expected return, namely \( \pi^* = \underset{\pi \in \Pi_C}{\arg\max} \; J(\pi) \).

\subsection{Model Predictive Control}
In this work, we use a centralized model predictive controller to minimize the cost function. We define $s^{t:k}$ and $\mathbf{a}^{t:k}$ as the global system state and input of all agents from time $t$ to $k$, respectively. The model predictive controller finds a set of optimal inputs ${\mathbf{a}^{(t+1:t+T)}}^{*}$ which minimize the cost function $C(\hat{s}^{t+1:t+T}, a^{t+1:t+T})$ over predicted state $ \hat{s} $ and control inputs $\mathbf{a}$ for some finite time horizon $T$:
\begin{equation}
\label{mpc_minimize}
{\mathbf{a}^{(t+1:t+T)}}^{*} = \underset{\mathbf{a}^{t+1:t+T}}{\mathrm{arg\,min}}\ C(\hat{s}^{t+1:t+T}, a^{t+1:t+T}).
\end{equation}
As the dynamics of multi-agent system is implicit in our setting, so we use a nonlinear function $f$ to predict future implicit states:
\begin{equation}
\label{constraint_function}
    \hat{s}^{t+1} = f({s}^t, \mathbf{a}^{t}) .
\end{equation}
We can apply this formula recurrently to predict further states until the end of the horizon $T$.
\section{METHODOLOGY}
\begin{figure}[h]
 \centering
 \includegraphics[width=0.95\linewidth]{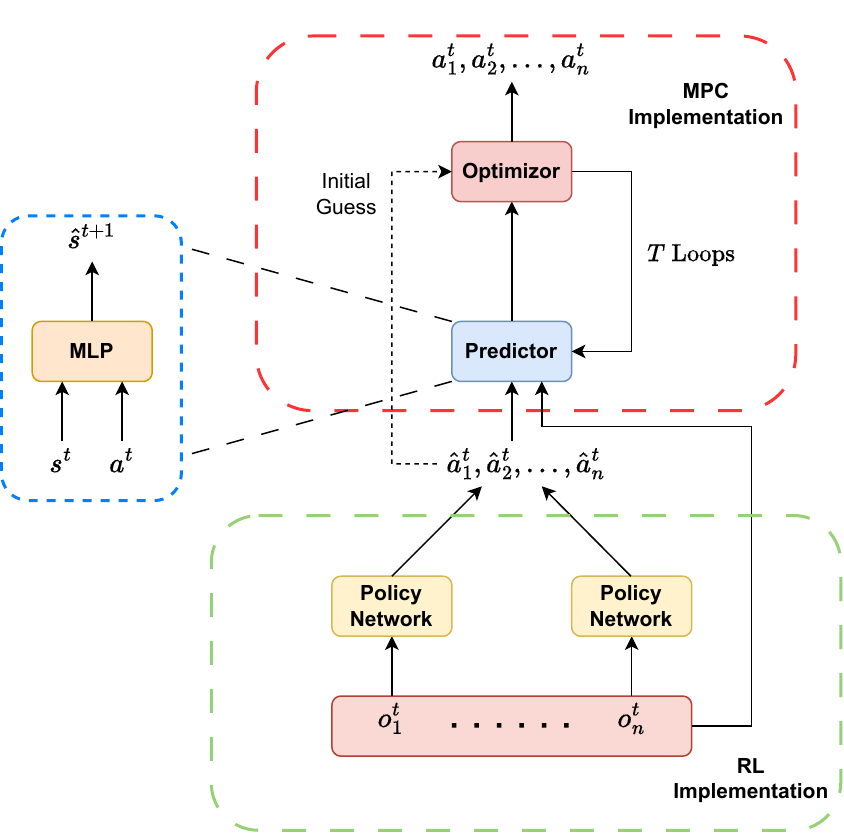}
 \caption{Implementation of DeepSafeMPC. This framework can be divided into RL and MPC parts. Within the RL domain, Policy Networks produce initial action vectors $\{\hat{a}^t_1, \hat{a}^t_2, \ldots, \hat{a}^t_n\} $. These vectors serve as preliminary inputs to the MPC's Predictor component and the initial guess for MPC optimizer. The Predictor, utilizing a Multi-Layer Perceptron (MLP), forecasts the forthcoming state \(\hat{s} ^{t+1} \) based on the current state \( s^t \) and action \( \mathbf{a}^t \). Subsequently, the Optimizer refines these actions into an optimized sequence  $\mathbf{a}^t = \{a^t_1, a^t_2, \ldots, a^t_n\}$ over the decision horizon \( T \). }
 \label{fig:model}
\end{figure}
Expanding on the above fundamental concepts of CMDP and MPC, DeepSafeMPC leverages MARL and MPC to cope with the reward maximization and cost reduction respectively. The whole process can be formulated as: 

\begin{align}
\underset{\pi \in \Pi}{\text{maximize}}
& \quad \mathbb{E}_{\mathbf{a} \sim \pi} \left[ \sum_{t=0}^{M} \gamma^t r(s^t, \mathbf{a}^t) \right], \label{eq:maximize} \\
\text{subject to}
& \quad C(\hat{s}^{t+1:t+T}, \mathbf{a}^{t+1:t+T}) \leq b. \label{eq:subject_to}
\end{align}
where the $b$ is the maximum cost the system can tolerate.

The algorithm progresses in a sequential two-step routine as shown in Figure \ref{fig:model}. Initially, the MAPPO algorithm generates a preliminary action. Then this action will be employed as the initial guess for MPC. MPC leverages a robust deep learning model to predict future states and iteratively refines the action by optimizing a predefined cost function \eqref{eq:subject_to}. This procedure ensures the robustness and safety of the decision-making.

\subsection{Multi-Agent Proximal Policy Optimization }
We now leverage MAPPO algorithm \cite{mappo} to maximize the return in \eqref{eq:maximize}.
MAPPO extends the Proximal Policy Optimization (PPO) algorithm \cite{ppo} to cooperative multi-agent settings, demonstrating remarkable efficiency and effectiveness in a variety of test environments. It adapts PPO's on-policy reinforcement learning approach, leveraging centralized training with decentralized execution to accommodate the complexities of multi-agent interactions.

MAPPO trains two separate neural networks: a decentralized actor network with parameters $\theta$, and a centralized critic function network with parameters $\phi$. This design is import for enabling the different patterns of training and execution to benefit from the respective advantages of centralized training and decentralized execution.

During each timestep $t$, agent $i$ receives a local observation $o_i^t$ and, utilizing the actor network, generates a corresponding action $\hat{a}_i^t$. The corresponding centralized critic network, on the other hand, aggregates observations from all agents to compute a collective critic value $v_i^t$, which serves to guide each agent's learning by providing a global perspective on the environment's state. While the action $\hat{a}_i^t$ is utilized across both execution and training phases, the critic value $v_i^t$ is exclusively employed during training to inform policy updates and enhance cooperative behavior among agents.

The actor network's training objective is designed to maximize the expected return, refined by an entropy term to encourage exploration and policy entropy \cite{soft_reinforcement}:

\begin{equation}
\begin{aligned}
\label{actor_mappo}
L_{\theta}(\theta)&=\left[ \frac{1}{Bn} \sum_{i=1}^{B} \sum_{t=1}^{n} \min \left(r_{\theta}^{t} A_i^{t}, \text{clip}\left(r_{\theta}^{t}, 1 - \epsilon, 1 + \epsilon\right) A_i^{t} \right) \right] \\&+ \sigma \left[ \frac{1}{Bn} \sum_{i=1}^{B} \sum_{t=1}^{n} Z[\pi_\theta(o_i^{t})] \right].
\end{aligned}
\end{equation}
where
$r_{\theta}^{t} = \frac{\pi_\theta(a_i^{t} | o_i^{t})}{\pi_{\theta_{\text{old}}}(a_i^{t} | o_i^{t})}, \ A_i^{t}$  is computed using the GAE \cite{gae} method, $\epsilon$ is the clip parameter, $ \sigma$ is the entropy coefficient hyperparameter and $Z$ is the policy entropy,  which encourages the exploration of new actions. 

In contrast, centralized critic networks aim to update their value function estimations by minimizing a carefully designed loss function. This function quantifies the discrepancy between estimated returns and actual observed returns, leveraging global information to achieve a comprehensive understanding of the environment. By assessing the collective information, the centralized critic provides a crucial feedback mechanism. This mechanism guides individual agents in optimizing their policies, ensuring that each agent's decisions are informed by not just their own experiences, but also by the collective experiences of all agents within the system. 
\begin{equation}
\begin{aligned}
\label{mappo_critic}
L_{\phi}(\phi) = \frac{1}{Bn} \sum_{i=1}^{B} \sum_{t=1}^{n} \max& (V_\phi (s_i^{t}) - \hat{R}_i, \text{clip} (V_\phi  (s_i^{t}), 
 V_{\phi_{\text{old}}}(s_i^{t}) \\& - \epsilon, V_{\phi_{\text{old}}}(s_i^{t}) + \epsilon) - \hat{R}_i )^2 ).
\end{aligned}
\end{equation}

where \(\hat{R}_i\) is the discounted reward-to-go.

In these formulations, $B$ represents the batch size and $n$ the number of agents. This detailed exploration of MAPPO highlights its sophisticated approach to multi-agent reinforcement learning, showcasing the critical roles of policy entropy and centralized critique in developing effective, adaptive cooperative strategies.

\subsection{Dynamics Predictor}

Given the initial actions generated by MAPPO and known system states, we use a deep learning-based MPC to predict the future states. We are currently at some timestep $t$, using the known system state $s^{t}$ and actions $\mathbf{a}^{t}$ as inputs. And our goal is then to predict the future system states $\hat{s}^{t+1:t+T} $ up to time-horizon $T$ by applying our model $\hat{s}^{t+1} = f({s}^t, \mathbf{a}^{t})$ recurrently.

During the training, we use an offline scheme to learn the set of parameters $\eta$. We gather extensive operational data states $s^{t}$,  actions $\mathbf{a}^{t}$ and next states ${s}^{t+1}$ from the environment, which is subsequently buffered. Then we leverage ${s}^{t+1}$ as the ground truth and minimize the error between the ground truth and generated next states $\hat{s}^{t+1}$. For the optimization objective, we utilize the Mean Squared Error (MSE) loss as follows:

\begin{equation}
    \label{mse_loss}
    \text{MSE Loss}= \frac{1}{N} \sum_{i=1}^{N} (\hat{s}^{t+1} - {s}^{t+1})^2 .
\end{equation}

By leveraging historical data, we construct a dataset that accurately captures the transitions between states under a range of  action conditions. This dataset enables a neural network to learn the temporal dynamics governing the evolution of states within the environment. To ascertain the robustness of the MPC's predictor, we provide a straightforward proof and demonstrate the prediction error through experimental results.

In the robustness analysis, our primary objective is to prove that the training error of the proposed method diminishes progressively over time. The training error is denoted by \(e^t = \hat{s}^t - s^t\), and we redefine the equation \eqref{constraint_function} as follows:
\begin{equation}
f_1(s^t, \mathbf{a}^t) = \sigma(\eta^*(s^t, \mathbf{a}^t)),  \label{eq:error_bound_1}
\end{equation}
\begin{equation}
f_2(s^t, \mathbf{a}^t) = \sigma(\eta(s^t, \mathbf{a}^t)). \label{eq:error_bound_2}
\end{equation}
Here, \(f_1(\cdot)\) and \(f_2(\cdot)\) are the same functions in \eqref{constraint_function}, \(\eta^*\) denotes the optimal predictor neural network weights, while \(\eta\) represents the actual predictor weights during training, and \(\sigma(\cdot)\) represents the activation function which means the outputs of \(f_1(\cdot)\) and \(f_2(\cdot)\) are bounded.
\begin{assumption}\label{ap:model_prediction}
The training samples and the neural network's weights \(\eta\) are bounded, and there exists a set of optimal weights \(\eta^*\).
\end{assumption}
\begin{lemma}
\label{lemma:1}
If the conditions of Assumption \ref{ap:model_prediction} hold, we have:
\begin{equation}
\| \eta^* - \eta \| < \varepsilon_w, \label{23}
\end{equation}
\begin{equation}
\| f_1(s^t, a^t) - f_2(s^t, a^t) \| = \| e(t) \| \leq \varepsilon_e, \label{24}
\end{equation}
where \(\varepsilon_w > 0\), \(\varepsilon_e > 0\), they represents the bound of distance between the optimal weights and actual weights and the distance between the outputs predictor generated by these two set of weights. Here, \(\| \cdot \|\) denotes the \(L_2\) norm.
\end{lemma}

\begin{proof}
Given that the weights \(\eta\) are bounded, optimal set of weights \(\eta^*\) exists. This implies that for the bounded sequence of weights in a finite-dimensional vector space, there is a finite maximum distance \(\varepsilon_w > 0\).
Due to the boundeness of \(\eta\) and the continuity of the neural network function, the neural network outputs for optimal and actual weights does not exceed a bound \(\varepsilon_e\), as stated in equation \eqref{24}.
\end{proof}

\begin{theorem}
With the deep learning-based predictive model in place, the training error \(e(t)\) is uniformly ultimately bounded (UUB).
\end{theorem}
\begin{proof}
We construct the Lyapunov function \cite{lstm} to analyze prediction errors:
\begin{equation}
    V(e(t)) = \frac{1}{2} e^T(t) e(t). 
\end{equation}
The derivative of \(V(e(t))\) is expressed as:
\begin{align*}
\label{25}
    \dot{V}(e(t)) &= e^T(t)\dot{e}(t) \\
    &= e^T(t) (f_1(s^t, \mathbf{a}^t) - f_2(s^t, \mathbf{a}^t) - e(t)) \\
    &= -\|e(t)\|^2 + \|e(t)\| \|\sigma(\eta^*(s^t, \mathbf{a}^t)) - \sigma(\eta(s^t, \mathbf{a}^t))\| \\
    &\leq -\|e(t)\|^2 + \|e(t)\| \|\eta^*(s^t, \mathbf{a}^t) - \eta(s^t, \mathbf{a}^t)\| \\
    &= -\|e(t)\|^2 + \|e(t)\| \varepsilon_w \\
    &= -\|e(t)\|(\|e(t)\| - \varepsilon_w).
\end{align*}
Above analysis suggests that:

When the error norm \(\|e(t)\|\) falls short of \(\varepsilon_w\), the derivative of the Lyapunov function becomes positive, indicating an escalation in the training error up to the point where \(\|e(t)\|\) reaches the threshold \(\varepsilon_w\). If the error norm equals or exceeds \(\varepsilon_w\), then \(\dot{V}(e(t))\) will become non-positive, leading to an error norm that is ultimately bounded by \(\varepsilon_e\).

In light of these findings, it can be concluded that the training error will asymptotically be less than or equal to \(\varepsilon_e\), which proves the robustness of the predictor model.
\end{proof}

\subsection{Integration with MPC}

By leveraging the deep learning-based predictor, we integrate it into MPC to cope with the implicit dynamics. The goal of MPC in this work is to find a suitable control signal $\mathbf{a}^t$ through online optimization. In this work, given prediction and control horizons $T$, the optimization problem of MPC can be formulated as follows:

\begin{equation}
\label{eq:mpc_optim}
\begin{aligned}
\underset{\mathbf{a}(t)}{\text{min}} \ &C(\hat{s}^{t+1:t+T}, \mathbf{a}^{t+1:t+T}), \\
\text{subject to}& \\
& \dot{\hat{s}}^{t+1} = f(s^t, \mathbf{a}^t), \\
& s_{\text{min}} \leq s^t \leq s_{\text{max}}, \\
& \hat{\mathbf{a}}_{\text{min}} \leq {\mathbf{a}}^t \leq \hat{\mathbf{a}}_{\text{max}}.
\end{aligned}
\end{equation}

For the optimization, we leverage Sequential Quadratic Programming (SQP) to solve the nonlinear optimization problem. 

First, we need to define the Lagrangian function of nonlinear programming:

\begin{equation}
\label{eq:lagrangian}
\begin{aligned}
L\left(s^{t}, \mathbf{a}^{t}, \mu, \nu, \xi\right) &= C\left(s^{t+1:t+T}, \mathbf{a}^{t+1:t+T}\right) \\
&+ \sum_{\tau=t+1}^{t+T} \lambda_{\tau} \left(f\left(s^{\tau}, \mathbf{a}^{\tau}\right) - \hat{s}^{\tau}\right) \\
&+ \sum_{\tau=t+1}^{t+T} \mu_{\tau} \left(s_{\min} - s^{\tau}\right) + \nu_{\tau} \left(s^{\tau} - s_{\max}\right) \\
&+ \sum_{\tau=t+1}^{t+T} \xi_{\tau} \left(\mathbf{a}_{\min} - \mathbf{a}^{\tau}\right) + \xi_{\tau} \left(\mathbf{a}^{\tau} - \mathbf{a}_{\max}\right).
\end{aligned}
\end{equation}

where $\lambda_{\tau}$ are the Lagrange multipliers for the equality constraints given by the system dynamics, $\mu_{\tau}$ and $\nu_{\tau}$ are the multipliers for the state constraints, and $\xi_{\tau}$ are the multipliers for the control action constraints.  The Karush-Kuhn-Tucker (KKT) optimality conditions of \eqref{eq:mpc_optim} are:

\begin{equation}
\label{eq:kkt}
\begin{aligned}
&\begin{bmatrix}
\nabla_{s} \mathcal{L}(s^{*}, \mathbf{a}^{*}, \mu^{*}, \nu^{*}, \xi^{*}, \zeta^{*}) \\
\nabla_{\mathbf{a}} \mathcal{L}(s^{*}, \mathbf{a}^{*}, \mu^{*}, \nu^{*}, \xi^{*}, \zeta^{*})
\end{bmatrix} 
\\
= &
\left[\begin{array}{l}
J_{1}\left(s^{*}\right)^{T}+J_{f}\left(s^{*}\right)^{T} \lambda^{*}-\mu^{*}+\nu^{*} \\
J_{2}\left(\mathbf{a}^{*}\right)^{T}+J_{f}\left(\mathbf{a}^{*}\right)^{T} \lambda^{*}-\xi^{*}+\zeta^{*}
\end{array}\right],\\
= &
\begin{bmatrix}
0 \\
0 
\end{bmatrix}.
\end{aligned}
\end{equation}

where $J_1(s)$ and $J_2(\mathbf{a})$ denote the Jacobian matrices     of $C(s, \mathbf{a})$ with respect to $s$ and $a$ respectively, and $J_f(s, \mathbf{a})$ denotes the Jacobian matrix of the system dynamics $f(s, \mathbf{a})$ with respect to $s$ and  $a$.

In SQP, each iteration involves solving a al Quadratic Programming (QP) subproblem to obtain a search direction for the optimization variables. The QP subproblem is formulated based on a quadratic approximation of the objective function and a linear approximation of the constraints around the current iterate $k$. And the QP subproblem for a particular iteration 
$(s^k, \mathbf{a}^k)$ can be stated as follows:

\begin{equation}
\label{eq:qp_sub}
\begin{aligned}
\min_{\Delta s, \Delta a} \quad & \frac{1}{2}
\begin{bmatrix}
\Delta s \\
\Delta \mathbf{a} 
\end{bmatrix}^{\mathrm{T}}
H^{k}
\begin{bmatrix}
\Delta s \\
\Delta \mathbf{a}
\end{bmatrix}
+ 
\begin{bmatrix}
\nabla_{s} C(s^{k}, \mathbf{a}^{k})^{\mathrm{T}} \\
\nabla_{\mathbf{a}} C(s^{k}, \mathbf{a}^{k})^{\mathrm{T}}
\end{bmatrix}^{\mathrm{T}}
\begin{bmatrix}
\Delta s \\
\Delta \mathbf{a}
\end{bmatrix},
\\
\text{subject to}\quad & J_f(s^k)^T \Delta s + J_f(\mathbf{a}^k)^T \Delta \mathbf{a} = y^{k+1} - f(s^k, \mathbf{a}^k), \\
& s_{\min} - s^k \leq \Delta s \leq s_{\max} - s^k, \\
& \mathbf{a}_{\min} - \mathbf{a}^k \leq \Delta \mathbf{a} \leq \mathbf{a}_{\max} - \mathbf{a}^k.
\end{aligned}
\end{equation}

where $H^{k}$ is an approximation to the Hessian of the Lagrangian, $\Delta s$ and $\Delta \mathbf{a} $ are the steps in the state and control variables, respectively.
$\nabla_s C(s^k, \mathbf{a}^k)$ and $\nabla_a C(s^k, \mathbf{a}^k)$ are the gradients of the cost function with respect to the state and control variables at iteration $k$.
The inequalities for $\Delta s$ and $\Delta \mathbf{a}$ ensure that the updated state and control variables remain within their respective bounds.

\begin{assumption}
    Assuming that the starting point is feasible and that the initial Hessian approximation $H^k$ is positive definite.
\end{assumption}

\begin{lemma}
\label{lemma:descent_direction}
The solution to the quadratic
programming (QP) subproblem \eqref{eq:qp_sub} yields a descent direction if the cost function $C(s,\mathbf{a})$ and dynamic function $f(s,\mathbf{a})$ are continuously differentiable.
\end{lemma}

\begin{proof}
Since $H^k$ is positive definite and $(\Delta s^*, \Delta \mathbf{a}^*)$ minimizes the quadratic model, the change in the objective function along $(\Delta s^*, \Delta \mathbf{a}^*)$ from $(s^k, \mathbf{a}^k)$ will be negative. Hence it is a descent direction:
\begin{equation}
\label{eq:descent_direction}
 \nabla C(s^k, \mathbf{a}^k)^\mathrm{T}[\Delta s^* \; \Delta \mathbf{a}^*] < 0.  
\end{equation}

\end{proof}

\begin{theorem}
If the linear independence constraint qualification (LICQ) holds at each iteration. The sequence ${(s^k,\mathbf{a}^k)}$ generated by the SQP algorithm converges to a stationary point of the original problem. And the accumulation point $(s^*,\mathbf{a}^*)$ is a local minimum.
\end{theorem}

\begin{proof}
    By Lemma \eqref{lemma:descent_direction} each iterate $(\nabla s,\nabla \mathbf{a})$ is a descent direction for the merit function.
    A line search along this direction yields a new iteration that reduces the merit function.
    Because of the descent property and the line search conditions, the sequence of iterates is bounded and therefore has an accumulation point.
    The accumulation point satisfies the KKT conditions \eqref{eq:kkt} due to the properties of the merit function and the descent method.  Since the sequence of iterations is bounded and the KKT residual is converging to zero, there exists an accumulation point $(s^*,\mathbf{a}^*)$  of the sequence that satisfies the KKT conditions and this point is the local minimum.

\end{proof}

\subsection{Implementation}
\begin{algorithm}
\caption{Training Process of DeepSafeMPC}
\begin{algorithmic}[1]
\State Initialize policy network parameters $\theta$, value function parameters $\phi$, predictor model parameters $\eta$
\State Initialize optimization algorithm for MPC with cost function $C$
\State Set learning rates $\alpha_\theta$, $\alpha_\phi$, $\alpha_\eta$
\State Set horizon $T$, episode length $L$, batch size $B$, maximum steps $I_{\max}$
\State Initialize data buffer $\mathcal{D}$, replay buffer $\mathcal{B}$
\For{each MAPPO training episode $e = 1,2,\ldots,E$}
    \State Collect trajectories by executing policy $\pi_\theta$
    \State Store state-action-reward sequences in $\mathcal{D}$ and $\mathcal{B}$
    \State Update policy $\pi_\theta$ and value function $V_\phi$ 
    \Statex \quad \; by (4) and (5)
    \State Clean the replay buffer $\mathcal{B}$
\EndFor
\For{each predictor model training step}
    \State Sample minibatch $D$ from $\mathcal{D}$
    \State Train predictor model $f_\eta$ by minimizing
    \Statex \quad \; MSE loss (6) on $D$
\EndFor
\Procedure{MPC Optimization with Predictor}{}
    \For{each control step $t = 1,2,\ldots,I_{\max}$}
        \State Get current state $s^t$ from the environment
        \State Generate initial action $\mathbf{a}^t$ using policy $\pi_\theta(s^t)$
        \State Initialize trajectory cost $J = 0$
        \For{$t = 0$ to horizon $T$}
            \State Predict next state $s^{t+1}$ using $f_\eta(s^t,\mathbf{a}^t)$
            \State Optimize $\mathbf{a}^{t+1}$ to minimize $C$ 
            \Statex \qquad \qquad \; using predictions from $f_\eta$
            \State Update trajectory cost $J += C(s^t, \mathbf{a}^t)$
        \EndFor
        \State Apply optimized action sequence to the system
    \EndFor
\EndProcedure
\end{algorithmic}
\end{algorithm}

So far we have discussed the components of our method,  we now explain the implementation of DeepSafeMPC in Fig.\ref{fig:model}. In the decentralized execution phase, we use the concatenated observations $\{o_1^{t}, o_2^{t}, ..., o_n^{t}\}$ to approximate the states $s^{t}$ of the environment.
The implementation process is shown in \ref{fig:model}. The training process is included in the Algorithm 1.





\section{EXPERIMENTAL RESULTS}

\begin{figure*}[h]
\centering
\subfigure[Two-Agent Ant]{
\includegraphics[width =  0.3\linewidth]{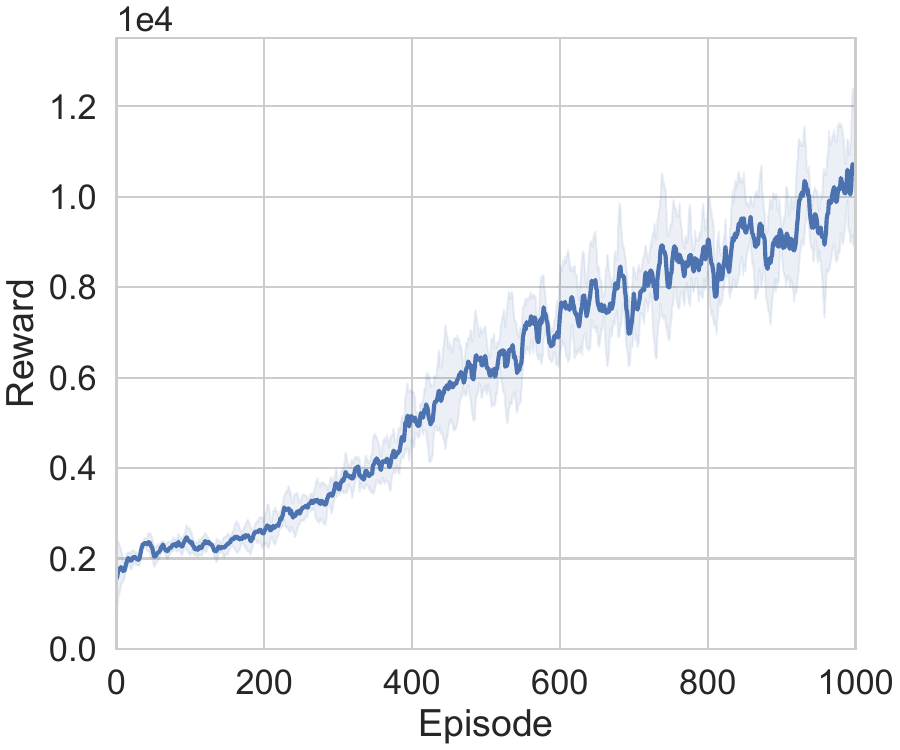}
\label{fig:reward_ant}
}
\subfigure[Half Cheetah]{
\includegraphics[width =  0.3\linewidth]{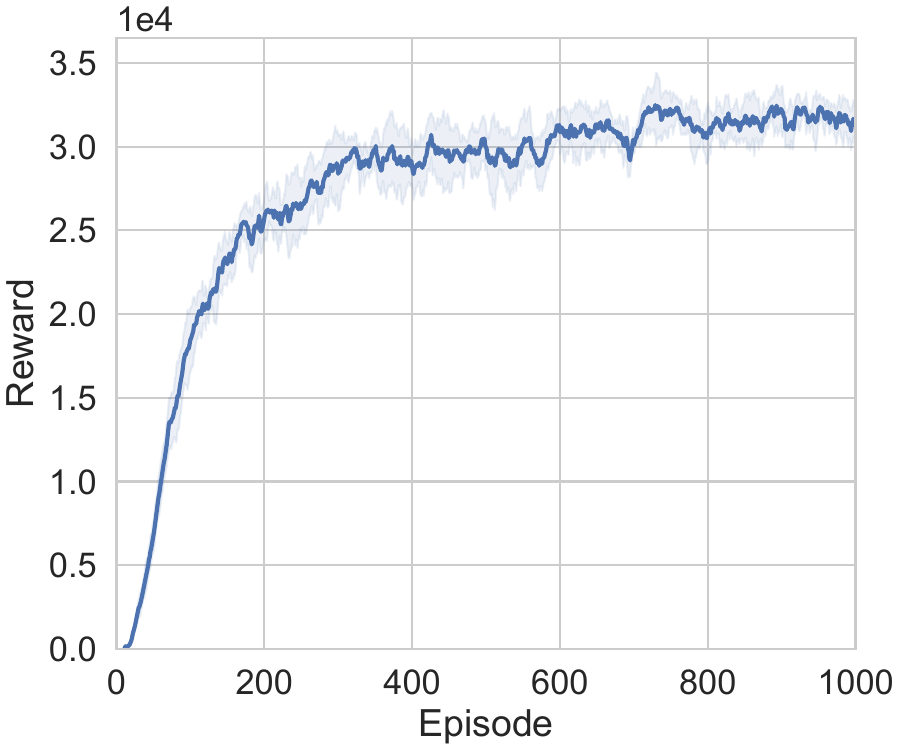}
\label{fig:reward_cheetah}
}
\subfigure[Swimmer]{
\includegraphics[width =  0.3\linewidth]{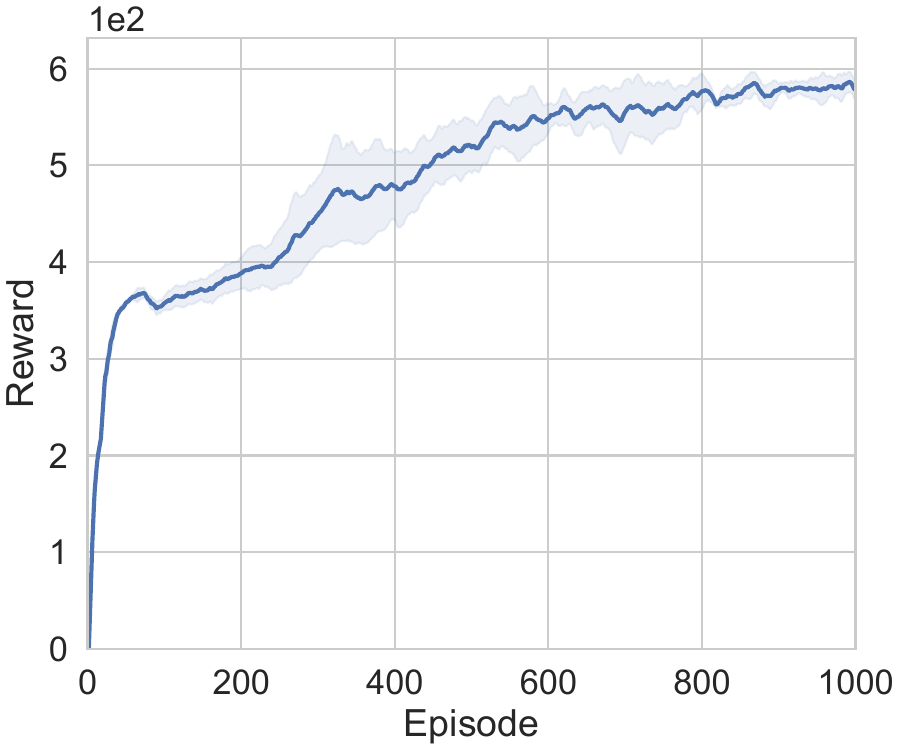}
\label{fig:reward_swimmer}
}
\caption{Experimental results: (a) Two-Agent Ant, (b) Half Cheetah, and (c) Swimmer.
}
\label{fig:Simulation_Results}
\end{figure*}

\subsection{Safe Multi-Agent MuJoCo}    
Safe Multi-Agent MuJoCo \cite{safemamujoco} is an extension of the multi-agent version of the MuJoCo simulator \cite{mamujoco, mujoco}. It provides a high-fidelity simulation environment for testing the efficacy of reinforcement learning algorithms in complex, continuous, and dynamic multi-agent scenarios. Built upon the renowned MuJoCo physics engine, it extends the single-agent benchmark tasks to the multi-agent domain, facilitating the study of cooperative, competitive, and mixed multi-agent interactions.

In our experiments subject to this environment, we configured several classic control and locomotion tasks to accommodate multiple agents. These tasks include environments such as Multi-Agent Ant, where each agent controls a subset of the limbs of a joint quadruped robot, and Multi-Agent Humanoid, which simulates a group of humanoid robots engaging in tasks that range from cooperative pushing of objects to competitive sumo wrestling.

\subsubsection{Metrics for Reinforcement Learning}
To evaluate the agents' performance, we define the metrics as the velocity of the x-axis speed. In our experiments, we conduct our algorithm on 2-Agent Swimmer, 2-Agent Ant and 2-Agent HalfCheetah tasks. The reward functions are defined as follows:

\begin{table}[ht]
\centering
\begin{tabular}{|c|c|c|}
\hline
Task & Reward Function \\
\hline
Swimmer & $\frac{\Delta x}{\Delta t} + 0.0001\alpha$ \\
\hline
Two-Agent Ant & $ \frac{\Delta x}{\Delta t} + 5 \cdot 10^{-4} \left\| \text{external contact forces} \right\|^2_2 + 0.5\alpha + 1 $ \\
\hline
HalfCheetah & $ \frac{\Delta x}{\Delta t} + 0.1\alpha$  \\
\hline
\end{tabular}
\caption{Reward Function of MAMuJoCo Tasks.}
\label{tab:reward_function}
\end{table}

\subsubsection{Cost Function}

In Safe MAMuJoCo, the cost encapsulates the potential for unsafe states and actions, grounding the reinforcement learning process in the realities of physical safety and operational integrity. The environment introduces a nuanced cost function that penalizes scenarios where the agent's actions may lead to states deemed unsafe, such as unsafe velocities. This paradigm shift from a sole focus on reward maximization to the inclusion of cost minimization necessitates a delicate balance. 
In our setting, we consider the scenarios as robots with time limits. For agents walking on a two-dimensional plane, the cost is calculated as:

\begin{equation}
\label{cost_function}
    C(s, a) = \sqrt{v_x^2 + v_y^2}.
\end{equation}

\subsection{Simulation Results}

\begin{figure}[h]
 \centering
 \includegraphics[width=0.6\linewidth]{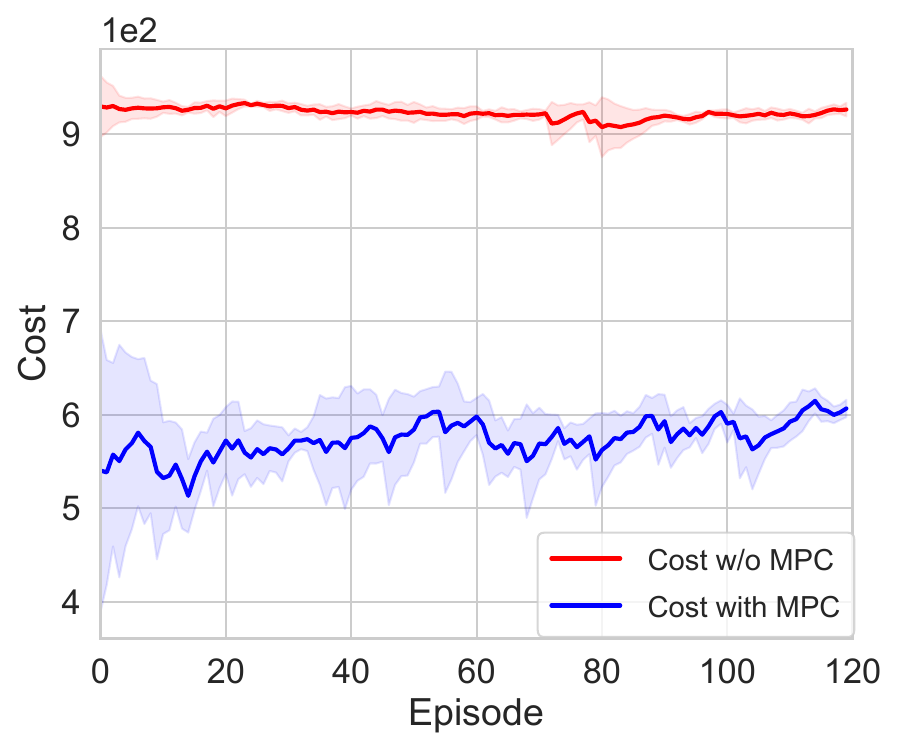}
 \caption{Comparison between with and w/o MPC.}
 \label{fig:mpc_cost}
\end{figure}

\begin{figure}[h]
 \centering
 \includegraphics[width=0.6\linewidth]{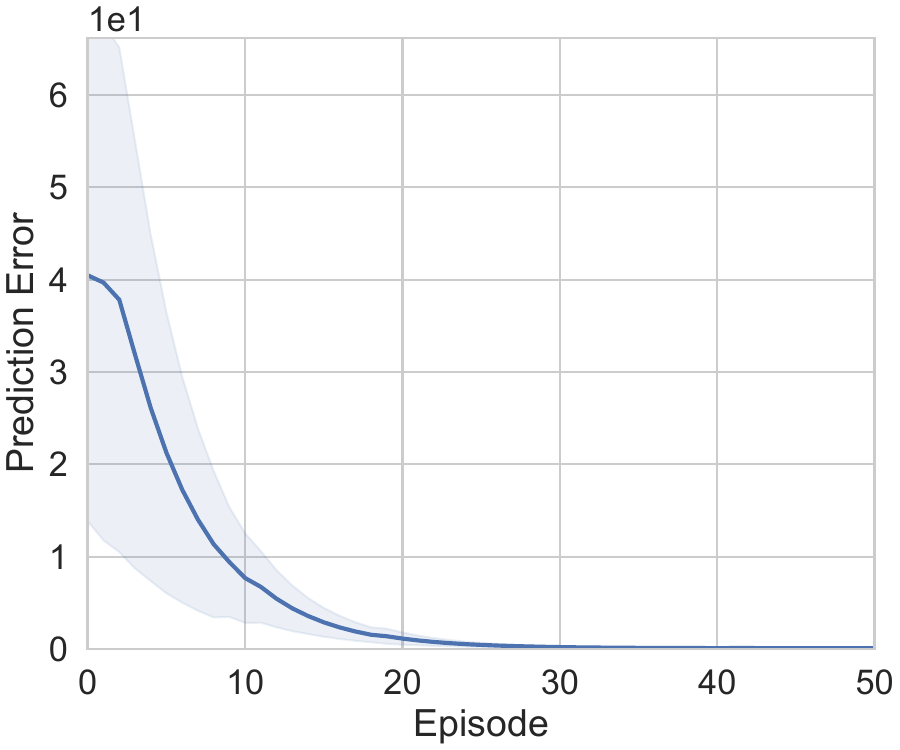}
 \caption{Prediction Error during Training.}
 \label{fig:prediction_error}
\end{figure}

The simulation prove the effectiveness of our algorithms. Figures \ref{fig:reward_ant}, \ref{fig:reward_swimmer} and  \ref{fig:reward_cheetah} present a comprehensive view of the MAPPO's performance over an extensive training steps. In these figures, we see a positive trend in the global reward, indicating a consistent improvement in the agents' ability to converge at 2-Agent Ant, Swimmer and Half Cheetah setting. Notably, the reward trajectory demonstrates a series of fluctuations, highlights the dynamic nature of the multi-agent environment and the agents' responses to the intrinsic stochasticity of the simulation.

Figure \ref{fig:mpc_cost} presents a comparative analysis of the costs incurred by MAPPO agents, both with and without the integration of an MPC controller, across a series of episodes. The data compellingly showcases the cost-reduction efficacy of the MPC controller, which reduces costs from over 900 to 600. The red line depicts the costs related to the actions taken by MAPPO agents without MPC adjustments, whereas the blue line illustrates the costs when actions are refined by the MPC controller. The shaded areas surrounding each line provides measures of costs.

Furthermore, Figure \ref{fig:prediction_error} offers an in-depth look at the prediction accuracy of our model. The depicted prediction error—measured as the Euclidean distance between the predicted next state $\hat{s}^{t+1}$ and the actual observed next state $s^{t+1}$ exhibits a downward trend throughout the training epochs. Early in the training, higher magnitudes of error are apparent, reflecting the initial calibration phase of the predictive model. The final results show that the error is limited within 0.0015, which testified the robustness of the dynamics predictor.

\begin{figure*}[h]
\centering
\subfigure[Two-Agent Ant]{
\includegraphics[width =  0.28\textwidth]{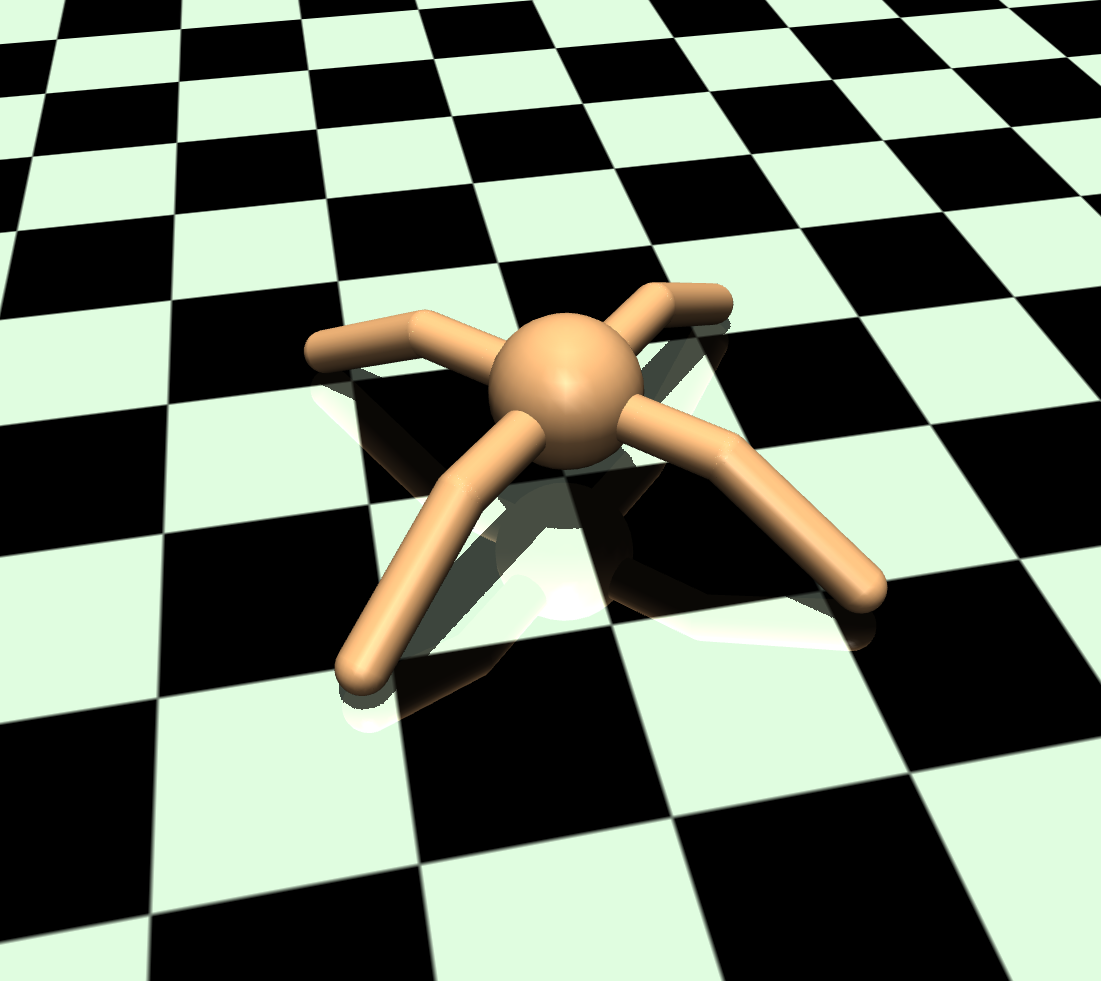}
\label{fig:ant}
}
\subfigure[Half Cheetah]{
\includegraphics[width =  0.28\textwidth]{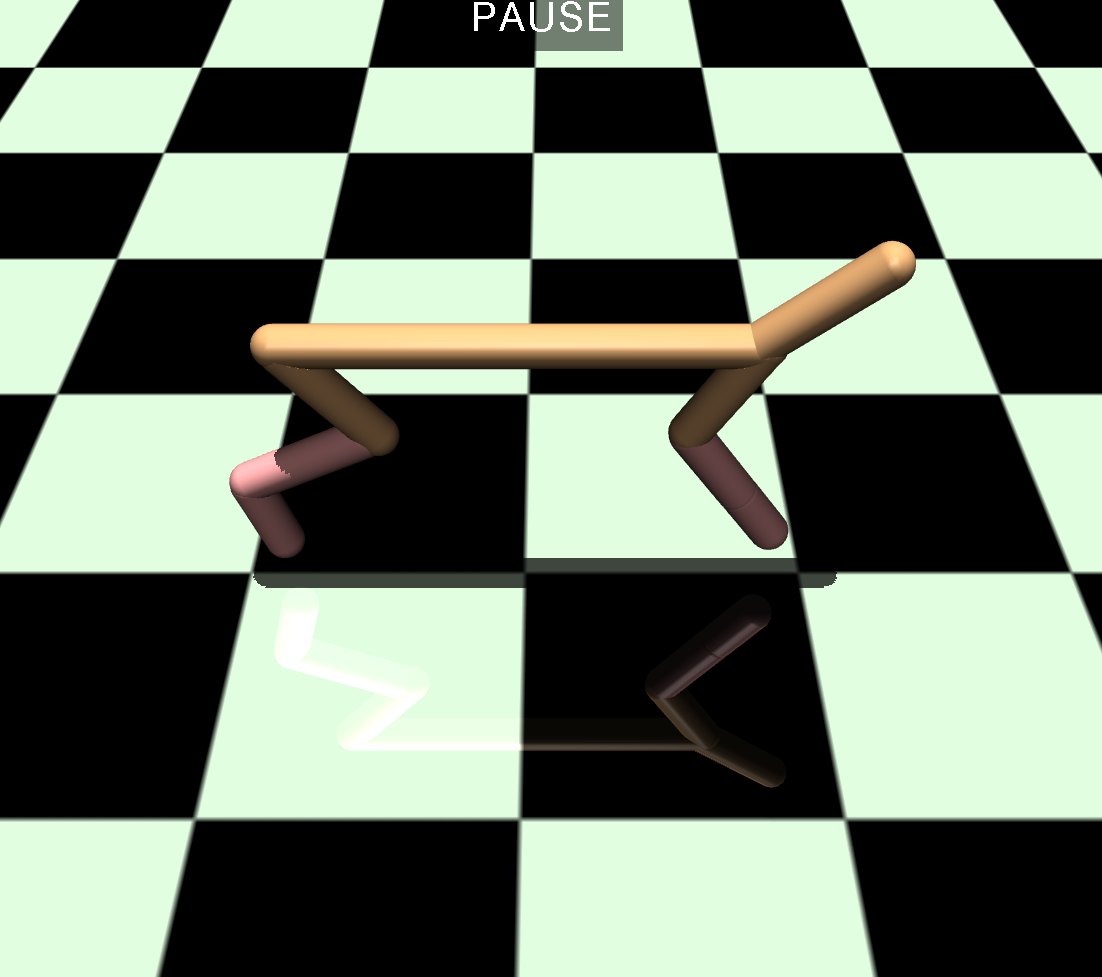}
\label{fig:cn}
}
\subfigure[Swimmer]{
\includegraphics[width =  0.28\textwidth]{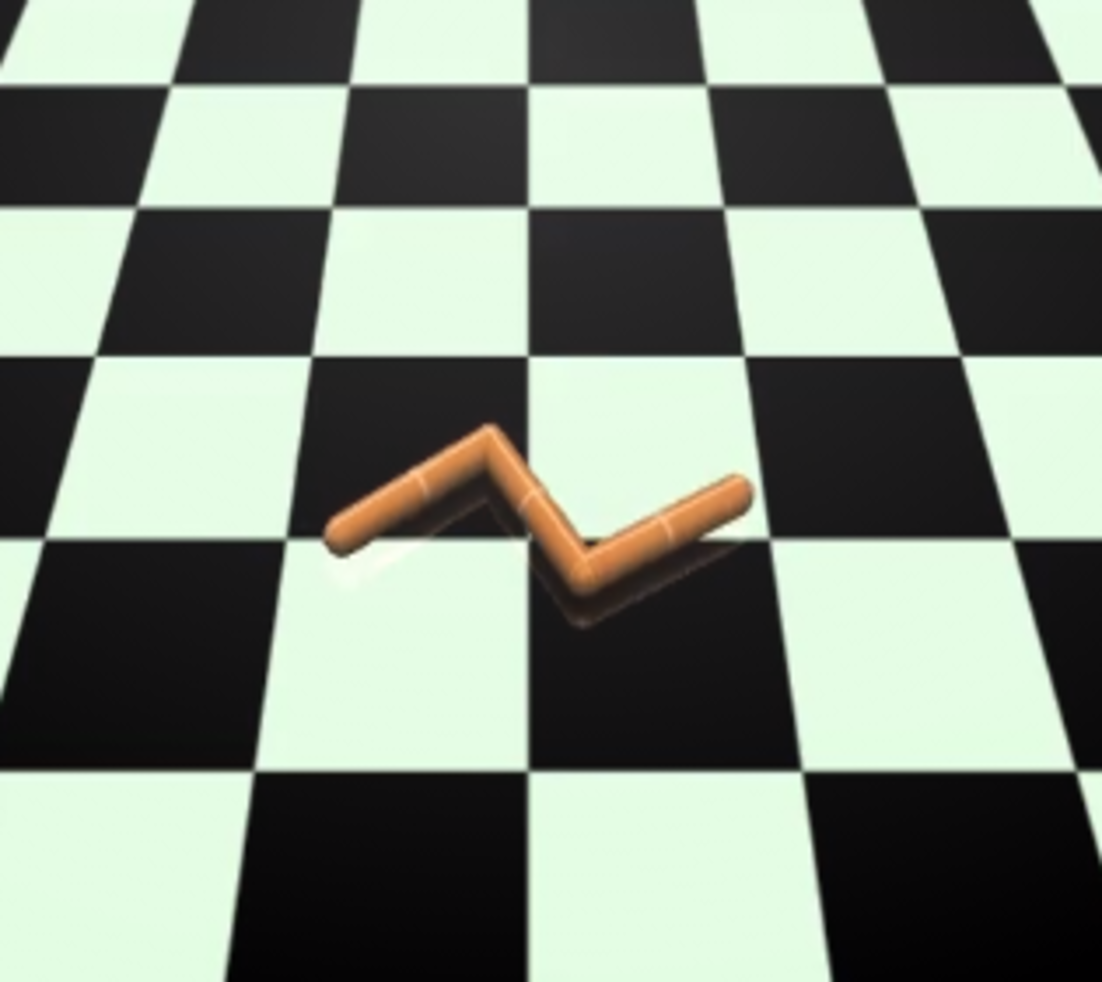}
\label{fig:pp}
}
\caption{Three environments in the experiments: (a) Two-Agent Ant, (b) Half Cheetah, and (c) Swimmer.
}
\label{fig:environments}
\end{figure*}



\section{CONCLUSIONS AND FUTURE WORK}
In this paper, we have presented DeepSafeMPC, an innovative approach that integrates deep learning with model predictive control to enhance safety in multi-agent environments. Our methodology leverages a robust dynamics predictor that facilitates accurate future state predictions by considering the current state and collective actions of all agents. Through a series of experiments in the Safe MAMuJoCo environment, DeepSafeMPC has shown promising results in managing the trade-off between task performance and operational safety. The experimental results highlight the adaptability and learning efficiency, as well as the system's capability to ensure safety.

Moving forward, DeepSafeMPC is poised for further development. A critical enhancement will be the implementation of safety constraints during the training phase of RL agents, not just execution, to refine the exploration process within safe boundaries. Additionally, future iterations will integrate uncertainty modeling into the predictor to bolster system robustness. Addressing these limitations will pave the way for deploying DeepSafeMPC in real-world applications with stringent safety requirements.



\section*{APPENDIX}
This appendix provides detailed descriptions of the neural network architectures and hyperparameters utilized in our experiments. We outline the configurations for the actor, critic, and predictor networks, alongside a comprehensive list of hyperparameters that govern their learning process.

\subsection{Neural Network Structures}
The structure of each neural network employed in our study is summarized in the table below. These architectures were designed to facilitate the learning process, with specific configurations for the actor, critic, and predictor networks. The table outlines the sequential arrangement of layers in each network, indicating the progression from input dimensions to output dimensions.

The learning process of the neural networks is guided by a set of hyperparameters, detailed in the table below. These parameters include both shared hyperparameters applicable to all networks and specific learning parameters unique to each model. This comprehensive listing ensures reproducibility and provides insights into the optimization strategies employed.

\subsection{Environment Descriptions}
\textbf{2x3 Half Cheetah:} This environment simulates a cheetah robot where two agents control three segments each. The objective is to maximize forward movement speed while adhering to safety constraints. The velocity of each agent is limited to ensure stability and safety, with a velocity threshold of 3.227.

\textbf{2x1 Swimmer:} In this setup, two agents each control one segment of a robotic swimmer. They work together to navigate through a fluid environment efficiently, with a velocity constraint of 0.04891 to prevent hazardous states.

\textbf{2x4 Ant:} This setup features an ant robot where two agents control four segments each. The goal is cooperative terrain navigation, with a safety constraint limiting velocity to 2.522 for safe interaction and movement.

The cost function is defined as an indicator function based on the agent's velocity exceeding a predefined threshold. Mathematically, it is represented as

\begin{equation}
\text{Cost} = 
\begin{cases} 
1 & \text{if velocity} > \text{Threshold} \\
0 & \text{otherwise}
\end{cases}.
\end{equation}

\begin{table}[htb]
\centering
\begin{tabular}{lccc}
\hline
Network Type & Layer 1 & Layer 2 & Output Layer \\ \hline
Actor        & input → 128 & 128 → 128 & 128 → output \\
Critic       & input → 128 & 128 → 128 & 128 → output \\
Predictor    & input → 64  & 64 → 64   & 64 → output  \\ \hline
\end{tabular}
\caption{Structure of Neural Networks}
\label{table:network-structures}
\end{table}

\begin{table}[htb]
\centering
\begin{tabular}{lc}
\hline
Hyperparameter       & Value     \\ \hline
Gamma ($\gamma$)            & 0.96      \\
GAE Lambda           & 0.95      \\
Target KL            & 0.016     \\
Searching Steps      & 10        \\
Accept Ratio         & 0.5       \\
Clip Parameter       & 0.2       \\
Learning Iters       & 5         \\
Max Grad Norm        & 10        \\
Huber Delta          & 10.0      \\
Actor LR             & 9.e-5     \\
Critic LR            & 5.e-3     \\
Optimizer Epsilon    & 1.e-5     \\
Weight Decay         & 0.0       \\ \hline
\end{tabular}
\caption{Shared and Learning Hyperparameters}
\label{table:hyperparameters}
\end{table}

Safety mechanisms integrated into these environments ensure that agents learn to avoid overspeeding, minimize the risk of collisions, and promote cooperative behavior among agents to achieve their objectives safely.

\subsection{Software and Hardware Specifications}
Experiments were conducted using the Gymnasium-Robotics framework and the MuJoCo physics engine on hardware comprising an Intel(R) Xeon(R) Gold 6254 CPU @ 3.10GHz and four NVIDIA A5000 GPUs. This setup ensures the computational efficiency and precision required for the demanding simulations involved in multi-agent reinforcement learning and safety evaluations.

\addtolength{\textheight}{-1cm}   

\bibliographystyle{unsrt}
\bibliography{cdc}

\begin{thebibliography}{10}

\bibitem{marl_survey}
Lucian Busoniu, Robert Babuska, and Bart De~Schutter.
\newblock A comprehensive survey of multiagent reinforcement learning.
\newblock {\em IEEE Transactions on Systems, Man, and Cybernetics, Part C (Applications and Reviews)}, 38(2):156--172, 2008.

\bibitem{ma_robotics}
Gregory Dudek, Michael~RM Jenkin, Evangelos Milios, and David Wilkes.
\newblock A taxonomy for multi-agent robotics.
\newblock {\em Autonomous Robots}, 3:375--397, 1996.

\bibitem{primal}
Guillaume Sartoretti, Justin Kerr, Yunfei Shi, Glenn Wagner, TK~Satish Kumar, Sven Koenig, and Howie Choset.
\newblock Primal: Pathfinding via reinforcement and imitation multi-agent learning.
\newblock {\em IEEE Robotics and Automation Letters}, 4(3):2378--2385, 2019.

\bibitem{alpha_Star}
Oriol Vinyals, Igor Babuschkin, Wojciech~M Czarnecki, Micha{\"e}l Mathieu, Andrew Dudzik, Junyoung Chung, David~H Choi, Richard Powell, Timo Ewalds, Petko Georgiev, et~al.
\newblock Grandmaster level in starcraft ii using multi-agent reinforcement learning.
\newblock {\em Nature}, 575(7782):350--354, 2019.

\bibitem{safety}
Pinxin Long, Tingxiang Fan, Xinyi Liao, Wenxi Liu, Hao Zhang, and Jia Pan.
\newblock Towards optimally decentralized multi-robot collision avoidance via deep reinforcement learning.
\newblock In {\em 2018 IEEE international conference on robotics and automation (ICRA)}, pages 6252--6259. IEEE, 2018.

\bibitem{safe_review}
Shangding Gu, Long Yang, Yali Du, Guang Chen, Florian Walter, Jun Wang, Yaodong Yang, and Alois Knoll.
\newblock A review of safe reinforcement learning: Methods, theory and applications, 2023.

\bibitem{nips_1}
Yiming Zhang, Quan Vuong, and Keith Ross.
\newblock First order constrained optimization in policy space.
\newblock {\em Advances in Neural Information Processing Systems}, 33:15338--15349, 2020.

\bibitem{nips_2}
Jiaming Ji, Borong Zhang, Jiayi Zhou, Xuehai Pan, Weidong Huang, Ruiyang Sun, Yiran Geng, Yifan Zhong, Josef Dai, and Yaodong Yang.
\newblock Safety gymnasium: A unified safe reinforcement learning benchmark.
\newblock {\em Advances in Neural Information Processing Systems}, 36, 2023.

\bibitem{nips_3}
Shangding Gu, Jakub~Grudzien Kuba, Munning Wen, Ruiqing Chen, Ziyan Wang, Zheng Tian, Jun Wang, Alois Knoll, and Yaodong Yang.
\newblock Multi-agent constrained policy optimisation, 2022.

\bibitem{non_stationary}
Georgios Papoudakis, Filippos Christianos, Arrasy Rahman, and Stefano~V. Albrecht.
\newblock Dealing with non-stationarity in multi-agent deep reinforcement learning, 2019.

\bibitem{lyapunov}
Felix Berkenkamp and Angela~P Schoellig.
\newblock Safe and robust learning control with gaussian processes.
\newblock In {\em 2015 European Control Conference (ECC)}, pages 2496--2501. IEEE, 2015.

\bibitem{robust_control}
Yousef Emam, Paul Glotfelter, Zsolt Kira, and Magnus Egerstedt.
\newblock Safe model-based reinforcement learning using robust control barrier functions.
\newblock {\em arXiv preprint arXiv:2110.05415}, 2021.

\bibitem{adaptive_control}
Said~G Khan, Guido Herrmann, Frank~L Lewis, Tony Pipe, and Chris Melhuish.
\newblock Reinforcement learning and optimal adaptive control: An overview and implementation examples.
\newblock {\em Annual reviews in control}, 36(1):42--59, 2012.

\bibitem{safe_mpc}
Mario Zanon and S{\'e}bastien Gros.
\newblock Safe reinforcement learning using robust mpc.
\newblock {\em IEEE Transactions on Automatic Control}, 66(8):3638--3652, 2020.

\bibitem{safe_mpc_2}
David~Q Mayne, James~B Rawlings, Christopher~V Rao, and Pierre~OM Scokaert.
\newblock Constrained model predictive control: Stability and optimality.
\newblock {\em Automatica}, 36(6):789--814, 2000.

\bibitem{deepmpc}
Ian Lenz, Ross~A Knepper, and Ashutosh Saxena.
\newblock Deepmpc: Learning deep latent features for model predictive control.
\newblock In {\em Robotics: Science and Systems}, volume~10, page~25. Rome, Italy, 2015.

\bibitem{mappo}
Chao Yu, Akash Velu, Eugene Vinitsky, Jiaxuan Gao, Yu~Wang, Alexandre Bayen, and Yi~Wu.
\newblock The surprising effectiveness of ppo in cooperative, multi-agent games, 2022.

\bibitem{adaptive}
Shankar Sastry, Marc Bodson, and James~F Bartram.
\newblock Adaptive control: stability, convergence, and robustness, 1990.

\bibitem{lynapnov_4}
Yinlam Chow, Ofir Nachum, Edgar Duenez-Guzman, and Mohammad Ghavamzadeh.
\newblock A lyapunov-based approach to safe reinforcement learning.
\newblock {\em Advances in neural information processing systems}, 31, 2018.

\bibitem{lynapnov_2}
Felix Berkenkamp, Matteo Turchetta, Angela~P. Schoellig, and Andreas Krause.
\newblock Safe model-based reinforcement learning with stability guarantees, 2017.

\bibitem{lynapnov_3}
Yinlam Chow, Ofir Nachum, Aleksandra Faust, Edgar Duenez-Guzman, and Mohammad Ghavamzadeh.
\newblock Lyapunov-based safe policy optimization for continuous control, 2019.

\bibitem{MPC_1}
Samuel Pfrommer, Tanmay Gautam, Alec Zhou, and Somayeh Sojoudi.
\newblock Safe reinforcement learning with chance-constrained model predictive control.
\newblock In {\em Learning for Dynamics and Control Conference}, pages 291--303. PMLR, 2022.

\bibitem{lstm}
Keke Huang, Ke~Wei, Fanbiao Li, Chunhua Yang, and Weihua Gui.
\newblock Lstm-mpc: A deep learning based predictive control method for multimode process control.
\newblock {\em IEEE Transactions on Industrial Electronics}, 2022.

\bibitem{mpc_learning}
Ugo Rosolia and Francesco Borrelli.
\newblock Learning model predictive control for iterative tasks. a data-driven control framework, 2017.

\bibitem{pomdp}
Frans~A Oliehoek and Christopher Amato.
\newblock {\em A concise introduction to decentralized {POMDPs}}.
\newblock Springer, 2016.

\bibitem{ppo}
John Schulman, Filip Wolski, Prafulla Dhariwal, Alec Radford, and Oleg Klimov.
\newblock Proximal policy optimization algorithms, 2017.

\bibitem{soft_reinforcement}
Tuomas Haarnoja, Aurick Zhou, Pieter Abbeel, and Sergey Levine.
\newblock Soft actor-critic: Off-policy maximum entropy deep reinforcement learning with a stochastic actor, 2018.

\bibitem{gae}
John Schulman, Philipp Moritz, Sergey Levine, Michael Jordan, and Pieter Abbeel.
\newblock High-dimensional continuous control using generalized advantage estimation.
\newblock {\em arXiv preprint arXiv:1506.02438}, 2015.

\bibitem{safemamujoco}
Shangding Gu, Jakub~Grudzien Kuba, Munning Wen, Ruiqing Chen, Ziyan Wang, Zheng Tian, Jun Wang, Alois Knoll, and Yaodong Yang.
\newblock Multi-agent constrained policy optimisation, 2022.

\bibitem{mamujoco}
Bei Peng, Tabish Rashid, Christian A.~Schroeder de~Witt, Pierre-Alexandre Kamienny, Philip H.~S. Torr, Wendelin Böhmer, and Shimon Whiteson.
\newblock Facmac: Factored multi-agent centralised policy gradients, 2021.

\bibitem{mujoco}
Emanuel Todorov, Tom Erez, and Yuval Tassa.
\newblock Mujoco: A physics engine for model-based control.
\newblock In {\em 2012 IEEE/RSJ International Conference on Intelligent Robots and Systems}, pages 5026--5033, 2012.

\end{thebibliography}

\end{document}